\documentclass{article}

\usepackage{fullpage}
\usepackage{hyperref}       
\usepackage{url}            
\usepackage{booktabs}       
\usepackage{amsfonts}       
\usepackage{nicefrac}       
\usepackage{microtype}      
\usepackage{xcolor}         
\usepackage{amsmath,amsfonts,amssymb,amsthm}
\usepackage{thm-restate}
\usepackage{multirow}
\usepackage{graphicx}
\usepackage{subfig}

\newcommand{\eps}{\epsilon}
\newcommand{\AboveThresh}{{\textsc{AboveThresh}}}
\newcommand{\GenAboveThresh}{{\textsc{GenAboveThresh}}}

\newcommand{\PrivSprt}{{\textsc{PrivSPRT}}}

\usepackage{xcolor}
\definecolor{DarkGreen}{rgb}{0.1,0.5,0.1}

\newcommand{\new}[1]{\textcolor{black}{#1}}
\newcommand{\newtext}[1]{\textcolor{black}{#1}}

\usepackage{algorithmicx,algpseudocode, algorithm}

\newtheorem{theorem}{Theorem}
\newtheorem{lemma}{Lemma}
\newtheorem{corollary}[lemma]{Corollary}
\newtheorem{definition}{Definition}

\begin{document}

\title{Private Sequential Hypothesis Testing for Statisticians:\\ Privacy, Error Rates, and Sample Size}

\date{}
%

\author{Wanrong Zhang\footnotemark[1] \and Yajun Mei\footnotemark[2] \and Rachel Cummings\footnotemark[3]}

\renewcommand{\thefootnote}{\fnsymbol{footnote}}
\footnotetext[1]{Harvard John A. Paulson School Of Engineering And Applied Sciences {\tt wanrongzhang@fas.harvard.edu}. Supported by a Computing Innovation Fellowship from the Computing Research Association (CRA) and the Computing Community Consortium (CCC). Part of this work was completed while W.Z. was at Georgia Institute of Technology.}
\footnotetext[2]{H. Milton Stewart School of Industrial and Systems Engineering, Georgia Institute of Technology. {\tt yajun.mei@isye.gatech.edu}. Supported in part by NSF-DMS grant 2015405. }
\footnotetext[3]{Department of Industrial Engineering and Operations Research, Columbia University. {\tt rac2239@columbia.edu}. Supported in part by  NSF grants CNS-1850187 and CNS-1942772 (CAREER), and a JPMorgan Chase Faculty Research Award. Part of this work was completed while R.C. was at Georgia Institute of Technology.}

\maketitle

\renewcommand{\thefootnote}{\arabic{footnote}}

\begin{abstract}
 	The {\em sequential hypothesis testing} problem is a class of statistical analyses where the sample size is not fixed in advance. Instead, the decision-process takes in new observations sequentially to make real-time decisions for testing an alternative hypothesis against a null hypothesis until some stopping criterion is satisfied. In many common applications of sequential hypothesis testing, the data can be highly sensitive and may require privacy protection; for example, sequential hypothesis testing is used in clinical trials, where doctors sequentially collect data from patients and must determine when to stop recruiting patients and whether the treatment is effective. The field of {\em differential privacy} has been developed to offer data analysis tools with strong privacy guarantees, and has been commonly applied to machine learning and statistical tasks.
 
 In this work, we study the sequential hypothesis testing problem under a slight variant of differential privacy, known as \emph{Renyi differential privacy}.  We present a new private algorithm based on Wald's Sequential Probability Ratio Test (SPRT) that also gives strong theoretical privacy guarantees. We provide theoretical analysis on statistical performance measured by Type I and Type II error as well as the expected sample size. We also empirically validate our theoretical results on several synthetic databases, showing that our algorithms also perform well in practice. Unlike previous work in private hypothesis testing that focused only on the classical fixed sample setting, our results in the sequential setting allow a conclusion to be reached much earlier, and thus saving the cost of collecting additional samples.

\end{abstract}


\section{Introduction}


Hypothesis testing is a fundamental task in statistics and machine learning, and involves testing a null hypothesis $H_0$ against an alternative hypothesis $H_1$, given observed data. For the usual statistical hypothesis tests, the sample size is fixed before the data are collected, but for a sequential test we observe \emph{streaming data}, where the total sample size depends on the data and is thus a random variable. Sequential hypothesis testing is valuable because it may enable a decision to be reached earlier than with a fixed sample size test, which is critical when waiting for additional samples is costly.

The most prominent algorithm for sequential hypothesis testing is the Sequential Probability Ratio Test (SPRT) initially developed by \cite{wald1945sequential} for efficient testing of anti-aircraft gunnery during World War II, and later used in the design of fully sequential clinical trials \cite{Arm50, Arm54}. This algorithm continuously monitors the log-likelihood ratio of the observed data under the alternative and under the null hypotheses, and halts as soon as this ratio takes a value that is either very large or very small, reflecting that one hypothesis is overwhelmingly more likely than the other, given the observed data. The analyst running SPRT can choose these thresholds to trade-off her desired confidence in her final decision with making decisions quickly (with respect to the number of samples). In modern day, SPRT and other techniques for sequential hypothesis testing are widely used for many real-world applications, including clinical trials and quality control~\cite{wald2004sequential, siegmund2013sequential, Whi97, ghosh1991handbook}.


Performance of a sequential testing procedure is evaluated using four main criteria: two operating characteristic (OC) functions to describe the accuracy of final decisions, and two average sample number (ASN) functions to describe how quickly a decision was reached. The two OC criteria are the probability of Type I error, $\Pr[\mbox{reject $H_0$} \; | H_0],$ and the probability of Type II error, $\Pr[\mbox{accept $H_0$}\; | H_1]$. Since the number of observations $T$ is a random variable, the two ASN functions are the expected sample size under the null and alternative hypotheses, $\mathbb{E}_{H_0}[T]$ and $\mathbb{E}_{H_1}[T]$, respectively. \cite{WW48} showed that the sequential probability ratio test (SPRT) is the optimal test of testing a simple null $H_0$ against a simple alternative $H_1$ when observations are assumed to be sampled i.i.d., where optimality is defined as simultaneously minimizing both $\mathbb{E}_{H_0}[T]$ and $\mathbb{E}_{H_1}[T]$ subject to constraints on Type I and Type II error probabilities.


In modern applications of sequential hypothesis testing --- for example to medical clinical trials --- privacy also becomes another crucial performance criterion, as the data and decisions can be highly sensitive. The field of {\em differential privacy} \cite{DMNS06} has emerged as the gold standard in private data analysis by providing algorithms with strong worst-case privacy guarantees.  It is a parameterized privacy notion, where the privacy parameter $\eps$ allows for a smooth tradeoff between accuracy of the analysis and privacy to the individuals in the database. Informally, an algorithm is $\eps$-differentially private if it ensures that any particular output of the algorithm is at most $e^\eps$ more likely when a single user's data are changed. In recent years, tools for differentially private data analysis have been deployed in practice by major organizations such as Google~\cite{ErlingssonPK14}, Apple~\cite{AppleDP17}, Microsoft~\cite{DingKY17}, and the U.S. Census Bureau~\cite{DajaniLSKRMGDGKKLSSVA17}. 

In this work, we provide the first differentially private algorithm for the sequential hypothesis testing problem with theoretical guarantees on the Type I and Type II error, and the expected sample size.
By focusing on the metrics most relevant to the field of statistics and its practitioners, our work may be more readily deployed in practice.
One real-world application of our results is the design of statistically valid sequential experiments and clinical trials before data are collected or observed. Typically when designing sequential experiments, a scientist must develop and pre-register a well-justified protocol for making final decisions under all possible data outcomes, and no further adjustments to the protocol can be made once data collection has begun. Fully sequential design of clinical trials, as suggested by \cite{Arm50, Arm54}, where evaluation occurs after each new patient outcome was not always possible for statistical or practical reasons -- e.g., it is difficult to convene a data and safety monitoring committee after each observation. With recent advancements in statistics and computing, it has become feasible to continuously monitor and evaluate every patient \cite{Whi97}. Modern examples of fully sequential trials include the ``MADIT'' clinical trial to evaluate the effect of an implanted defibrillator \cite{Dem98} and a COVID-19 therapeutics trial intended to speed up the decision process \cite{Har20}. Fully sequential trials risk leaking patient's sensitive information, especially for patients with data collected shortly before the trial is halted. Our proposed private sequential test can be used for monitoring trials where privacy protection is necessary, such as those with irreversible clinical outcomes like death or severe infectious disease. It can also balance the tradeoff between small expected sample sizes for rapid decision, controlled Type I and Type II error properties, and formal privacy protections.

\subsection{Our contribution}

In this work, we combine tools from differential privacy with classical statistical methods for sequential hypothesis testing to develop a private version of Wald's SPRT, which we call \PrivSprt.

The most natural existing tool for privatizing Wald's SPRT is a private subroutine called \AboveThresh\ \cite{DNRRV09, dwork2014algorithmic} (also known as \textsc{SparseVector}). This algorithm takes in a database $X$ and a stream of queries $q_1,q_2,\ldots$, and sequentially privately tests whether the numerical value of each query $q_i$ evaluated on the database $q_i(X)$ is above or below a pre-specified threshold. A natural first attempt at a private version of SPRT would be to instantiate \AboveThresh\ using the SPRT test statistic as the query and using the SPRT stopping criteria as the threshold (see Section \ref{s.sprtprelim} for more details). However, as we show in Section \ref{sec:exp}, the random noise internal to \AboveThresh\ that is used to guarantee privacy causes extremely poor performance in terms of the relevant OC and ASN metrics.  In particular, we note that while \AboveThresh\ was designed to provide good performance with respect to high-probability finite-sample performance guarantees that are commonly used in the computer science literature, it fails to provide good performance on the metrics that are most relevant to the statistics community, such as Type I and Type II error.


We instead build our algorithm \PrivSprt\ using a generalized version of \AboveThresh\ from \cite{zhu2020improving} instantiated with Gaussian noise (rather than Laplace as in \cite{DNRRV09}), and we show that this modification results in good performance in terms of the OC and ASN metrics of interest. Specifically, we give bounds on the expected sample size of \PrivSprt\ (Theorem \ref{thm.sample}) and the Type I and Type II error (Theorem \ref{thm:errorrate}). We analyze the privacy of \PrivSprt\ through a generalization of DP known as \emph{Renyi differential privacy} (RDP) \cite{mironov2017renyi}, which is often preferred in practice due to its tighter composition properties with Gaussian noise \cite{WBK19}. We show that \PrivSprt\ satisfies RDP (Theorem \ref{thm.priv}), which also implies that is satisfies DP (Theorem \ref{thm.rdptodp}). Finally, we perform experiments to empirically validate our theoretical findings (Section \ref{sec:exp}).




\subsection{Related work}




Background on non-private SPRT was presented earlier in this section, so we focus our attention here on private hypothesis testing.
Private (fixed-sample-size) hypothesis testing has previously been considered in the static setting, where the analyst wishes to test a hypothesis (or family of hypothesis) at a single point in time for a fixed database~\cite{gaboardi2016differentially, gaboardi2018local, sheffet2018locally, couch2019differentially, canonne2019structure}.  Dynamic or online private sequential decision making has recently gained traction in various settings, including recent work on private sequential change-point detection \cite{cummings2018differentially, cummings2019privately, zhang2021single}. These works all rely on the \AboveThresh/\textsc{SparseVector} technique to achieve privacy in sequential change-point problems, where the focus is on the privacy of parameter estimation of change-point. Our work deals with the sequential hypothesis testing problem which is essentially a classification problem, and our aim is to provide a unifying approach by showing that a generalization of this technique can be applied to solve general private sequential hypothesis testing problems for a more general class of accuracy objectives. \newtext{ 
\cite{wang2020differential} considers privatization of SPRT. Their algorithm is to add Laplace noise to the thresholds to generate a noisy stopping time, and then use exponential mechanism to output the binary decision. They show that the algorithm can provide a weaker notion of privacy that is data dependent, and it will only converge to DP when the stopping time goes to $\infty$. In contrast, our results aim to minimize stopping time, and therefore, a direct comparison would not be applicable.}

\section{Preliminaries}
This section provides the background on sequential hypothesis testing (Section \ref{s.sprtprelim}) and the differentially private tools (Section \ref{s.dpprelim}) that will be brought to bear in our \PrivSprt\ algorithm.

\subsection{Sequential hypothesis testing}\label{s.sprtprelim}

A sequence $X$ of data points, $x_1, x_2, \cdots,$ are observed sequentially, i.e., arriving one at a time.  Let $f_t(x_1,\ldots,x_t)$ denote the true joint probability density function (pdf) of the first $t$ observations, $(x_1, x_2, \ldots, x_{t})$. Under the simplest model where the data points are sampled i.i.d. from some distribution $f$, then $f_t(x_1,\ldots,x_t) = \prod_{i=1}^{t} f(x_i)$. In more general dependence models, $f_t(x_1,\ldots,x_t) = \prod_{i=1}^{t} f(x_i | x_{1}, \cdots, x_{i-1})$.

In \emph{sequential hypothesis testing} problems, the analyst has two possible hypotheses on the pdfs -- $f_{0t}$ and $f_{1t}$ -- and her goal is to quickly (i.e., with as few samples as possible) and correctly test the null hypotheses $H_0: f_t=f_{0t}$ against the alternative $H_1: f_t=f_{1t}$.\footnote{For simplicity in the remainder of this paper, we will abuse notation to use the subscripts $0$ and $1$ to indicate probability with respect to the distributions given in $H_0$ and $H_1$, respectively.}  At each time $t$, the analyst must make one of the following three decisions: (1) halt collecting observations and \emph{accept} the null hypothesis $H_0$, (2) halt collecting observations and \emph{reject} the null hypothesis $H_0,$ or (3) continue collecting observations to provide additional information.

There are four main criteria to assess the performance of sequential tests, including two operating characteristic (OC) functions and two average sample number (ASN) functions \cite{wald2004sequential,siegmund2013sequential}. The two OC functions are Type I error, $\Pr_0[ \text{reject } H_0]$ (i.e., rejecting $H_0$ when $H_0$ is true), and Type II error, $\Pr_1[ \text{accept } H_0]$ (i.e., accepting $H_0$ when $H_1$ is true), which address correct decision-making, and are well-studied in the standard classification or hypothesis testing contexts. The two ASN functions are the expected sample size under both the null and alternative hypotheses, i.e.,  $\mathbb{E}_0[T]$ and $\mathbb{E}_1[T]$, which ensure that decisions are made efficiently and that unnecessary costs are not incurred by collecting too many samples. In sequential hypothesis testing problems, the objective is to simultaneously minimize $\mathbb{E}_0[T]$ and $\mathbb{E}_1[T]$ subject to the constraints that Type I and Type II error probabilities are both small.

Wald's sequential probability ratio test (SPRT) \cite{wald1945sequential} is a celebrated optimal solution when testing a simple null $H_0$ \new{, where the joint distribution is completely specified,} 
against a simple alternative $H_1$ under the simplest i.i.d.~model, where the data are independent and identically distributed. The idea behind SPRT is straightforward: the analyst continues to collect observations until she has enough evidence to confidently decide whether $H_0$ or $H_1$ is true, as measured by the cumulative log-likelihood ratio statistic being either too large or too small. Mathematically, at each time $t$, the analyst calculates the cumulative log-likelihood ratio statistic:
$\ell_t=\log \frac{f_{1t} (x_1,\ldots, x_t)}{f_{0t} (x_1,\ldots, x_t)}.$
Under the i.i.d.~model, this test statistic becomes:
$\ell_t=\sum_{i=1}^t\log \frac{f_1(x_i)}{f_0(x_i)}.$
Moreover, the analyst chooses two positive constants $a, b$, and runs the SPRT test until the following stopping time is reached:
$T=\min\{t\ge 1: \ell_t\notin (-a,b)\}.$
After reaching the stopping criterion, a statistical decision is made based on the following rule:
\begin{align*}
\text{Reject }H_0 \quad &\text{ if }\ell_T\ge b,\\
\text{Accept }H_0 \quad &\text{ if } \ell_T\le -a.
\end{align*}
Intuitively, the set $(-a,b)$ is the range of test statistics where the analyst is uncertain between $H_0$ and $H_1$. If the test statistic ever falls outside of this range, then the analyst can have high confidence about one of the hypotheses being true. Under the i.i.d.~model, the SPRT is exactly optimal in the sense of minimizing both expected sample sizes, $\mathbb{E}_0[T]$  and $\mathbb{E}_1[T]$, simultaneously, among all other (sequential or fixed-sample size) tests whose Type I and Type II error probabilities are same as (or smaller than) those of the SPRT \cite{WW48}. Below we denote $x(a) \approxeq  y(a)$ if $x(a)/ y(a) \rightarrow 1$ if $a \rightarrow \infty$ (or if $a \rightarrow 0$).



\begin{theorem}[Error Rates \cite{wald1945sequential}]
The approximation 
of Type I error of SPRT is $\Pr_0[\ell_T\ge b] \approxeq \frac{1-exp(-a)}{\exp(b)-\exp(-a)}$, 
and the approximation of the Type II error of SPRT is
$\Pr_1[\ell_T\le -a] \approxeq \exp(-a)\frac{\exp(b)-1}{\exp(b)-\exp(-a)}$.
\end{theorem}

The additional assumption that the observations $x_t$ are independent and identically distributed is required to give the expected sample size. 
\begin{theorem}[Expected Sample Size \cite{wald1945sequential}]
When $x_1,x_2,\ldots$ are sampled i.i.d., SPRT has expected samples sizes: 
\begin{equation}\label{nonpriv.size1}
\mathbb{E}_1[T]\approxeq \frac{-a\exp(-a)(\exp(b)-1)+b\exp(b)(1-\exp(-a))}{D_{KL}(f_1||f_0)(\exp(b)-\exp(-a))},
\end{equation}
\begin{equation}\label{nonpriv.size2}
\mathbb{E}_0[T]\approxeq \frac{-a(\exp(b)-1)+b(1-\exp(-a))}{-D_{KL}(f_0||f_1)(\exp(b)-\exp(-a))}.
\end{equation}
\end{theorem}


\subsection{Differential privacy}\label{s.dpprelim}
Differential privacy is a statistical notion of database privacy, which ensures that the output of an algorithm will still have approximately the same distribution is a single data entry were to be changed. Differential privacy considers a general database space $\mathcal{D}$. If databases are real-valued and contain a fixed number $n$ of entries, then $\mathcal{D} = \mathbb{R}^n$; in our sequential hypothesis testing setting, our database will be of a random size so $\mathcal{D} = \mathbb{R}^*$. Two databases $X, X' \in \mathcal{D}$ are said to be \emph{neighboring} if they differ in at most one entry.

\begin{definition}[Differential Privacy \cite{DMNS06}]\label{def.dp}
	A randomized algorithm $\mathcal{M}: \mathcal{D} \rightarrow \mathcal{R}$ is \emph{$(\epsilon,\delta)$-differentially private} if for every pair of neighboring databases $X,X' \in \mathcal{D}$, and for every subset of possible outputs $\mathcal{S} \subseteq \mathcal{R}$,
	$\Pr[\mathcal{M}(X) \in \mathcal{S}] \leq \exp(\epsilon)\Pr[\mathcal{M}(X') \in \mathcal{S}] + \delta.$
\end{definition}

Renyi differential privacy (RDP) is a relaxation of differential privacy based on the Renyi divergence, defined as $D_\alpha(P||Q)=\frac{1}{\alpha-1}\log \mathbb{E}_Q \left( \frac{P(x)}{Q(x)}\right)^\alpha$. This privacy notion requires that the distribution over outputs on two neighboring databases is close in Renyi divergence.

\begin{definition}[Renyi Differential Privacy \cite{mironov2017renyi}]\label{def.rdp}
	A randomized algorithm $\mathcal{M}: \mathcal{D} \rightarrow \mathcal{R}$ is $(\alpha,\eps)$-RDP with order $\alpha\ge 1$, if for neighboring datasets $X,X' \in \mathcal{D}$ it holds that
$D_\alpha(\mathcal{M}(X)||\mathcal{M}(X'))\le \eps.$
\end{definition}

Renyi differential privacy is desirable for its straightforward \emph{composition}, meaning that the privacy parameters degrade gracefully as additional computations are performed on the data, even when the private mechanisms are chosen adaptively. This allows us to design RDP mechanisms using simple private building blocks.

\begin{theorem}[Basic RDP Composition \cite{mironov2017renyi}]
	Let $\mathcal{M}_1: \mathcal{D} \rightarrow \mathcal{R}$ is $(\alpha,\eps_1)$-RDP and $\mathcal{M}_2: \mathcal{D} \rightarrow \mathcal{R}$ is $(\alpha,\eps_2)$-RDP, then the mechanism defined as $(\mathcal{M}_1, \mathcal{M}_2)$ satisfies $(\alpha, \eps_1+\eps_2)$-RDP.
\end{theorem}

While DP also satisfies its own variant of composition, RDP is especially amenable to composition of Gaussian noise mechanisms. We can also easily translate between the notions of RDP and DP because any $(\alpha,\eps)$-RDP mechanism is also $(\eps_\delta,\delta)$-differential privacy for $\delta>0$, as shown below in Theorem \ref{thm.rdptodp}. Thus when running multiple RDP mechanisms, a common approach is to first perform RDP composition across the mechanisms and then translate the RDP guarantee into one of differential privacy.

\begin{theorem}[From RDP to DP \cite{mironov2017renyi}]\label{thm.rdptodp}
	If $\mathcal{M}$ is $(\alpha,\eps)$-RDP, then it is also $(\eps+\frac{\log 1/\delta}{\alpha-1}, \delta)$-differential privacy for any $0<\delta<1$.
\end{theorem}

Mechanisms for achieving both privacy notions typically add noise that scales with the \emph{sensitivity} of the function being evaluated, which is the maximum change in the function's value between two neighboring databases. For a real-valued function $q$, this is formally defined as: $\Delta q = \max_{X,X' \text{ neighbors}} | q(X) - q(X')|$.

The Gaussian mechanism with parameters $(\eps, \delta, \sigma)$ takes in a function $q$, database $X$, and outputs $q(X)+\mathcal{N}(0,\sigma^2)$. The scale of the noise is fully specified as $\sigma=\sqrt{2\log(1.25/\delta)}\Delta q/\eps,$  given the privacy parameters $\eps$ and $\delta$ and the query sensitivity $\Delta q$. 

\begin{theorem}[Privacy of Gaussian Mechanism \cite{dwork2014algorithmic}]\label{thm.dpgaussian}
	The Gaussian Mechanism with parameter $\sigma= \sqrt{2\log(1.25/\delta)}\Delta q/\eps$ is $(\eps,\delta)$-differentially private.
\end{theorem}

The \AboveThresh\ algorithm \cite{DNRRV09, dwork2014algorithmic} is a DP mechanism for handling a sequence of queries arriving online. It takes in a potentially unbounded stream of queries, compares the answer of each query to a fixed noisy threshold, and halts when it finds a noisy answer that exceeds the noisy threshold \newtext{(denoted as $\top$, and otherwise $\bot$)}, where the added noise follows the Laplace distribution. In many cases, more concentrated noise (e.g., Gaussian) is preferred, and \cite{zhu2020improving} gives the generalized version of \GenAboveThresh\ (presented in Algorithm \ref{alg.svt}), using general noise-adding mechanisms $\mathcal{M}_1$ and $\mathcal{M}_2$. These mechanisms can be any RDP algorithms that take in a real-valued input and produce a noisy estimate of the value. Our algorithm \PrivSprt\ will rely on an instantiation of \GenAboveThresh\ using Gaussian mechanisms for differential privacy. 

{\centering
	\begin{minipage}{\linewidth}
		\begin{algorithm}[H]
			\caption{Generalized Above Noisy Threshold: \GenAboveThresh($X, \Delta, \{q_1, q_2, \ldots \}, H, \mathcal{M}_1, \mathcal{M}_2$) }
			\begin{algorithmic}
				\State \textbf{Input:} database $X$, stream of queries $\{q_1, q_2, \ldots \}$ each with sensitivity $\Delta$, threshold $H$, noise-adding mechanisms $\mathcal{M}_1, \mathcal{M}_2$ that each add noise to their real-valued input.
				\State Let $\hat{H} \sim \mathcal{M}_1(H)$
				\For {each query $i$}
				\State Let $\hat{q}_i \sim \mathcal{M}_2(q_i(X))$
				\If {$\hat{q}_i>\hat{H}$}
				\State Output $a_i=\top$
				\State Halt
				\Else
				\State Output $a_i=\bot$
				\EndIf
				\EndFor
			\end{algorithmic}\label{alg.svt}
		\end{algorithm}
	\end{minipage}
}

\begin{theorem}[Privacy of \GenAboveThresh\ \cite{zhu2020improving}.] \label{thm.privsvt}
	Let $\mathcal{M}_1$ be any private mechanism that satisfies $\eps_1(\alpha)$-RDP for queries with sensitivity $\Delta$, and $\mathcal{M}_2$ be any private mechanism that satisfies $\eps_2(\alpha)$-RDP for queries with sensitivity $2\Delta$. Let $T$ be a random variable indicating the stopping time of Algorithm \ref{alg.svt} instantiated with ($X, \Delta, \{q_1, q_2, \ldots \}, H, \mathcal{M}_1, \mathcal{M}_2$). 
	Then Algorithm  \ref{alg.svt}  (denotes by $\mathcal{M}$) satisfies
	\begin{equation}\label{priv.svt}
	D_\alpha(\mathcal{M}(X)||\mathcal{M}(X'))\le \eps_1(\alpha)+\eps_2(\alpha) +\tfrac{\log\sup \mathbb{E}[T|Z_1]}{\alpha-1}, 
	\end{equation}
	and
	\begin{align}
	D_\alpha(\mathcal{M}(X)||\mathcal{M}(X'))\le & \frac{\alpha-(\gamma-1/\gamma)}{\alpha-1}\eps_1(\frac{\gamma}{\gamma-1}\alpha)+\eps_2(\alpha) \notag \\ &+ \frac{\log \mathbb{E}_{Z_1} (\mathbb{E}[T|Z_1]^\gamma)}{\gamma(\alpha-1)}, \label{priv.svt2}
	\end{align}
	for all $\gamma>1$ and $1<\alpha<\infty$, where $Z_1$ is the added noise from $\mathcal{M}_1$.
\end{theorem}
In the case where the expected length is bounded by $t_{\text{max}}$, Theorem \ref{thm.privsvt} implies an RDP bound of the form $\eps_1(\alpha)+\eps_2(\alpha) +\log(1+t_{\text{max}})/(\alpha-1)$.

%

\section{Private Sequential Hypothesis Testing}
In this section, we present our main result, which is a differentially private algorithm for the sequential hypothesis testing problem that also has small expected sample size and low Type I and Type II errors. We present our \PrivSprt\ algorithm in Section \ref{sec:algo} and the theoretical results on privacy, error rates, and sample size in Section \ref{sec:thm}.

\subsection{\PrivSprt\ algorithm}\label{sec:algo}

We present our algorithm for private sequential hypothesis testing, \PrivSprt, given formally in Algorithm \ref{alg.privsprt}. The algorithm is a private version of SPRT, and it uses two parallel instantiations of \GenAboveThresh\ to ensure privacy of the statistical decision. It instantiates two Gaussian mechanims with parameters $\sigma_1$ and $\sigma_2$ as the noise-adding mechanisms, $\mathcal{M}_1$ and $\mathcal{M}_2$, respectively. At each time $t$, the algorithm computes the log-likelihood ratio $\ell_t$ for $x_1,x_2, \ldots, x_t$, and uses the Gaussian mechanism to add noise to the log-likelihood ratio.
It then compares this noisy statistic against two pre-fixed noisy thresholds that depend on the SPRT decision thresholds $a$ and $b$, and the other Gaussian mechanism with parameter $\sigma_1$. The stopping condition of \PrivSprt\ is similar to that of SPRT, only using noisy versions of the thresholds. Once the stopping condition is reached, the algorithm stops collecting additional samples and outputs its statistical decision.

It is useful to highlight that we add noises to the cumulative log-likelihood ratio statistics, will allow us to maintain the first-order statistical optimality of our proposed algorithms. Here, the first-order optimality means the expected sample sizes of our algorithms subject to the privacy constraints converge to the classical optimal non-private expected sample size results up to $O(1)$. Meanwhile, we should mention that 
one could also add noises individual log-likelihood ratio statistics to satisfy the privacy constraints, but doing so will  severely affect the expected sample sizes, and thus yield to algorithms that are suboptimal from the statistical efficiency viewpoint.

The sensitivity of the log-likelihood ratios is defined as:
$\Delta(\ell)=\max_x \log \frac{f_1(x)}{f_0(x)}-\min_{x'}\log \frac{f_1(x')}{f_0(x')}.$
For certain distributions, including Gaussians, the sensitivity $\Delta(\ell)$ is unbounded and therefore would require infinite noise to preserve privacy. We instead use a truncation parameter $A>0$ to control the sensitivity of the log-likelihood ratio calculation, and add noise proportional to the post-truncation range. We note that the idea of truncating the likelihood  for privacy also appears in \cite{canonne2019structure} for private simple hypotheses testing and \cite{zhang2021single} for private sequential change-point detection. The $A$-truncated log-likelihood ratio is
\[\ell_t(A) = \textstyle\sum_{i=1}^t [ \log \tfrac{f_1(x_i)}{f_0(x_i)}]^{A}_{-A},\]
where the truncation operation is defined as $[x]^{A}_{-A}= -A, \text{if } x<-A; A, \text{if }  x>A; x, \text{otherwise}.$

{\centering
	\begin{minipage}{\linewidth}
		\begin{algorithm}[H]
			\caption{Private Sequential Probability Ratio Test: PrivSPRT($X, f_1, f_2 , -a, b, \sigma_1, \sigma_2, A$) }
			\begin{algorithmic}
				\State \textbf{Input:} database $X$, distributions $f_0, f_1$, SPRT thresholds $-a, b$, Gaussian mechanisms $\mathcal{M}_1$ with parameter $(\eps'/2, \delta, \sigma_1)$ and $\mathcal{M}_2$ with parameters $(\eps'/2, \delta, \sigma_2)$, truncation parameter $A$
				\State Let $\hat{-a} \sim \mathcal{M}_1(-a)$ and $\hat{b} \sim \mathcal{M}_1(b)$
				\For {each time $t$}
				\State Compute $\ell_t(A)=\sum_{i=1}^t \left[ \log \frac{f_1(x_i)}{f_0(x_i)}\right]^{A}_{-A}$
				\State Let $\hat{\ell_t^a} \sim \mathcal{M}_2(\ell_t(A))$ and $\hat{\ell_t^b} \sim \mathcal{M}_2(\ell_t(A))$
				\If {$\hat{\ell_t^b}>\hat{b}$}
				\State Halt and output $d=1$ (reject $H_0$)
				\ElsIf {$\hat{\ell_t^a}<\hat{-a}$}
				\State Halt and output $d=0$ (accept $H_0$)
				\Else
				\State Proceed to the next iteration
				\EndIf
				\EndFor
			\end{algorithmic}\label{alg.privsprt}
		\end{algorithm}
	\end{minipage}
}

\paragraph{Comparing to standard \AboveThresh.}
One may wonder why \GenAboveThresh\ is needed, and whether the original \AboveThresh\ algorithm of \cite{DNRRV09, dwork2014algorithmic} with Laplace noise (as referred to as \emph{\textsc{SparseVector}}) would be sufficient, perhaps with some loss in accuracy. In fact, this change to Laplace noise would break the desirable statistical properties of (non-private) SPRT. \new{ The properties of the SPRT depends on the overshoot of $\ell_{T}-b$ or $\ell_{T}-(-a)$, and to maintain the first-order optimality on the expected sample size, controlling the second moments of the noisy statistics is necessary. Adding Laplace noise will make the variance too large, and thus the desirable properties will break down.}
Empirically, we show in Section \ref{sec:exp} that using Laplace noise instead of Gaussian noise results in undesirable performance. 
On the theoretical side, statistical analysis of the SPRT is traditionally based on renewal theory and overshoot analysis in applied probability, which both rely heavily on the central limit theorem (CLT), and thus the standard techniques are still applicable when adding Gaussian noise for privacy.  On the other hand, if we add Laplace noise, the standard statistical techniques are inapplicable to characterize the overshoots; it remains an open problem to develop new tools to analyze the corresponding statistical properties.



\subsection{Theoretical results on privacy, sample size, and error rates}\label{sec:thm}

In this subsection, we provide formal results on the privacy guarantees and statistical properties of \PrivSprt. For analyzing the expected sample size, we will relate $\mathbb{E}_0[T]$ and $\mathbb{E}_1[T]$ to the input parameters $a,b$. Similarly for analyzing the error rates, we will relate the Type I and Type II error to $a$ and $b$. Recall that these errors respectively correspond to the false positive and false negative rates of the algorithm, which can be respectively defined as $\alpha=\Pr_{0}[\ell_t(A)+Z_t^b\ge b+Z_b]$ and $\beta=\Pr_{1}[\ell_t(A)+Z_t^a\le a+Z_a]$ from \PrivSprt. 

While our statistical properties of sample size and error rate are analyzed under the assumption that $x_1, x_2, \ldots$ follow either $H_0$ or $H_1$, as is standard in the statistics literature, our privacy guarantees hold unconditionally, regardless of the actual data distribution.


\paragraph{Privacy.}
Privacy of \PrivSprt\ follows by composition of two parallel instantiation of Algorithm \ref{alg.svt}, one each for the upper and lower bounds on $\ell_t$. Theorem \ref{priv.svt} gives Renyi divergence bounds for the outputs on two neighboring databases for \GenAboveThresh, but it only implies Renyi differential privacy when the conditional expectation \new{of the stopping time} or the moments of conditional expectation \new{of the stopping time} are bounded.
\cite{zhu2020improving} shows that the stopping time of \GenAboveThresh\ instantiated with Gaussian noise and non-negative queries has bounded moments of the conditional expectation \new{of the stopping time},  and thus it satisfies RDP.  However, in our case, the log-likelihood ratio queries can be negative, and this result cannot be immediately applied in our setting. Therefore, to prove that \PrivSprt\ is private, we must show that the expectation of the stopping time $T$ is bounded. 
\newtext{We remark that we can alternatively halt Algorithm \ref{alg.privsprt} when $t$ reaches an upper bound, and then make a decision using hypothesis testing methods when the sample size is fixed. However, this approach requires new analysis for the sample size and error rates \cite{siegmund2013sequential}. }
The full proof of Theorem \ref{thm.priv} appears in Appendix \ref{app.privacy}.

\begin{restatable}[Privacy]{theorem}{privacy}\label{thm.priv}
Let $T_A= E_{Z_A}[1+\rho_1^{-1}+\frac{5(a+Z_A)+3\sqrt{2}\sigma_2}{2\mu_0}]^\gamma$ and $T_B= E_{Z_B}[1+\rho_0^{-1}+\frac{5(b+Z_B)+3\sqrt{2}\sigma_2}{2\mu_1}]^\gamma$, where $\rho_0=1-\exp(-\frac{(1-c)\mu_1^2}{2A^2})$ and $\rho_1=1-\exp(-\frac{(1-c)\mu_0^2}{2A^2})$. Then algorithm \ref{alg.privsprt} satisfies $(\alpha, \frac{\alpha\gamma/(\gamma-1)-1}{\alpha-1}\frac{2\alpha A^2}{\sigma_1^2}+\frac{4\alpha A^2}{\sigma_2^2} + \frac{2\log \max\{T_A,T_B\}}{\gamma(\alpha-1)})$-RDP, for any $1<\alpha<\infty$.
\end{restatable}

\begin{corollary}\label{coro.priv}
	For $\sigma_1$ and $\sigma_2$ are chosen to be the parameters specified in the Gaussian mechanisms that satisfy $(\eps'/2, \delta)$-differential privacy, \PrivSprt$(X,f_1,f_2,-a,b,\sigma_1,\sigma_2,A)$ in Algorithm \ref{alg.privsprt} satisfies $(\alpha, (\frac{\alpha\gamma/(\gamma-1)-1}{\alpha-1}+1)\eps'+\frac{2\log \max\{T_A,T_B\}}{\alpha-1})$-RDP, for any $1<\alpha<\infty$.
\end{corollary}

\newtext{Because we are using the Gaussian mechanism as the noise-adding mechanism, the dependence of the stopping time in the privacy guarantee is unavoidable. $T_A$ and $T_B$ in Theorem \ref{thm.priv} are the moments of the conditional expectation of the stopping time, which depend on the true underlying distribution that generated the data; $T_A$ is roughly $O((\frac{a+\sigma_2}{\mu_0})^\gamma)$, and similarly $T_B$ is roughly $O((\frac{b+\sigma_2}{\mu_1})^\gamma)$. }
Theorem \ref{thm.priv} and Corollary \ref{coro.priv} further imply an $(\eps,\delta)$-differential privacy bound for \PrivSprt\ by Theorem \ref{thm.rdptodp}. For $\delta<1/2\log\max\{T_A,T_B\}$, and $\sigma_1$ and $\sigma_2$ are chosen to be the parameters specified in the Gaussian mechanisms that satisfy $(\eps'/2, \delta)$-differential privacy,  \PrivSprt\ is $(\eps, \delta)$-differentially private, with $\eps= (\frac{\alpha\gamma/(\gamma-1)-1}{\alpha-1}+1)\eps'+4\log(1/\delta)/(\alpha-1)$.

\paragraph{Sample Size.}
When analyzing statistical properties of \PrivSprt, an important quantity is the expectation of the truncated individual log-likelihood ratios:
\begin{align*}
\mu_0=-\mathbb{E}_0[ \log \tfrac{f_1(x)}{f_0(x)}]^{A}_{-A} , \; \; \text{ and } \; \;
\mu_1=\mathbb{E}_1[ \log \tfrac{f_1(x)}{f_0(x)}]^{A}_{-A}.
\end{align*}
When $A$ goes to $\infty$, the above expectations converge to the KL-divergence between $f_0$ and $f_1$.

A technical challenge that arises in bounding the expected sample size is that the noisy log-likelihood ratio at time $t$ cannot be decomposed into a summation of $t$  i.i.d.~random variables \new{because of the noise terms.} This preludes the use of Wald's identity \cite{wald1944cumulative}, which is used in the proof of bounded sample size for non-private SPRT, and relates the expectation of a sum of randomly-many finite-mean, i.i.d.~random variables to the expected number of terms in the sum and the expectation of the random variables.

Instead, we leverage a critical fact that $\mathbb{E}_i[T]=\sum_{t=0}^\infty \Pr_i[T>t]$ for $i\in\{0,1\}$, and thus relate the expected sample size to the probability of the noisy truncated log-likelihood ratio being within the noisy thresholds at each time $t$. Since the event is less probable for a large $t$, we partition the range $[0, \infty)$ into several sub-intervals, and bound the probability in each sub-interval seperately. This results in our $O(\frac{b}{\mu_1})$ bound on the expected sample size in Theorem \ref{thm.sample} when the noise parameters $\sigma_1$ and $\sigma_2$ go to $0$. This result is consistent with the non-private sample size result $O(\frac{b}{D_{KL}})$, and it is first-order optimal.  We note that a similar idea of partitioning the whole range into sub-intervals also appears in \cite{liu2020improved}, where it was applied only for handling Gaussian data.

The last term in the bound of Theorem \ref{thm.sample} is the additional cost that comes from adding Gaussian noise, which quantifies the cost of privacy. In the proof,
we permit large values of the difference between the Gaussian noise $Z_t$ to $Z_b$ (or $Z_a$) for a large $t$, which reduces the additional expected sample size required for privacy. \new{The analysis relies on partitioning the range into k intervals and a time-specific threshold depending on a constant $c$, and the results are under the optimal choice of $k$ and $c$.}  The proof is given in Appendix \ref{app.ssproof}.


\begin{restatable}[Sample Size]{theorem}{samplesize}\label{thm.sample}
	The expected sample size of \PrivSprt$(X,f_1,f_2,-a,b,\sigma_1,\sigma_2,A)$ under $H_1$ satisfies
	$\mathbb{E}_1[T] \le 1 +\min_{k \in \mathbb{N}}\min_{c\in (0,1)}(\frac{b}{(1-c)\mu_1}+\frac{1}{2(k+1)}\frac{b}{(1-c)\mu_1}+(k+1)\rho_1^{-1}+ \frac{3\sqrt{2(\sigma_1^2+ \sigma_2^2)}}{4(1-c)\mu_1} ),$
	where $\rho_0=1-\exp(-\frac{(1-c)\mu_1^2}{2A^2})$.
	Similarly, the expected sample size under $H_0$ satisfies
	$\mathbb{E}_0[T] \le 1 +\min_{k \in \mathbb{N}}\min_{c \in (0,1)} (\frac{a}{(1-c)\mu_0}+\frac{1}{2(k+1)}\frac{a}{(1-c)\mu_0}+(k+1)\rho_0^{-1}+ \frac{3\sqrt{2(\sigma_1^2+\sigma_2^2)}}{4(1-c)\mu_0} ),$
	where $\rho_1=1-\exp(-\frac{(1-c)\mu_0^2}{2A^2})$.
\end{restatable}

\newtext{To interpret the results in Theorem \ref{thm.sample}, we choose a specific (potentially suboptimal) values of $k$ and $c$. Choosing $k=1$ and $c=\frac{1}{2}$ gives $\mathbb{E}[T]\le 1+\rho_1^{-1}+\frac{5b}{2\mu_1}+\frac{3\sqrt{2(\sigma_1^2+\sigma_2^2)}}{2\mu_1}$, which is $O(\frac{b}{\mu_1})+O(\frac{\sqrt{(\sigma_1^2+\sigma_2^2)}}{\mu_1})$. The first term is the same as in the classical non-private results, and the second term is the additional cost for privacy. Since $\sigma_1$ and $\sigma_2$ will be chosen to scale with $\frac{A}{\epsilon'}$, the additional cost for privacy is $O(\frac{A}{\mu_1\epsilon})$, where $\mu_1$ is the expectation of the truncated log-likelihood ratios, which serves as a distance measure similar to the KL divergence. Our expected sample size and error rate results converge to the classical non-private results up to $O(1)$, ignoring the dependence on $\epsilon$. The asymptotic dependence on $\epsilon$ is $O(1/\epsilon)$, which matches the sample complexity dependence on $\epsilon$ in the simpler problem of private simple hypothesis testing \cite{canonne2019structure}. }


\paragraph{Error rates.}
We now move to provide guarantees for the Type I and Type II error rates of \PrivSprt. In the classical sequential hypothesis testing literature for non-private SPRT, the standard technique to characterize the error rates is based on the change of measure method that heavily utilizes the likelihood ratio statistics. Unfortunately, the test statistics of \PrivSprt\ are no longer the likelihood ratio, since the algorithm add Gaussian noise and truncates the log-likelihood for privacy. As a result, the standard change-of-measure technique is no longer applicable.


To characterize the error rates of \PrivSprt, we apply an alternative method based on the brute force estimation of the error probabilities, which was first proposed in \cite{sahu2015distributed} in the context of distributed hypothesis testing in sensor networks. It turns out that this alternative method is also applicable to the setting of \PrivSprt. The main idea is as follows: Type I error, $\Pr_0[d=1]$, can be written as a sum of probabilities of the noisy log-likelihood ratio being above the noisy threshold at time $t$ \new{and the event that the stopping time is $t$} for all $t>0$: $\sum_{t=1}^\infty \Pr_0[\ell_t(A)+Z_t>b+Z_b \; \land \; T=t]$.
We then partition the range of time $[1,\infty)$ into several sub-intervals and analyze them separately as before with the expected sample size. Although the high-level approach is similar to analyzing the expected sample size, the sub-intervals need to be carefully chosen here to give a meaningful bound for the error rates. The detailed proof is deferred to Appendix \ref{app.errorproof}.

\begin{restatable}[Error Rate]{theorem}{errorrate}\label{thm:errorrate}
	Let $d\in\{0,1\}$ be the decision output by \PrivSprt$(X,f_1,f_2,-a,b,\sigma_1,\sigma_2,A)$. Then the Type I error is bounded by:
	\begin{align}\label{error1}
	Pr_0[d=1]\le \min_{k\in\mathbb{N}} \min_{c\in(0,1)} \left\{ Q_1+Q_2+Q_3\right\},
	\end{align}
	where $Q_1=2\rho_0^{-1}\exp(-\frac{2b(1-c)\mu_0}{A^2}) (1+k\exp(\frac{1}{8k})$, $Q_2=k\exp(\frac{1}{4k+3})) $ and $Q_3=\frac{\sqrt{2(\sigma_1^2 +\sigma_2^2)}}{4(1-c)\mu_0} $, and $\rho_0=1-\exp(-\frac{(1-c)\mu_0^2}{2A^2})$. The Type II error is bounded by:
	\begin{equation}
	Pr_1[d=0]\le \min_{k\in\mathbb{N}} \min_{c\in(0,1)} \left\{ W_1+W_2+W_3\right\},
	\end{equation}	
	where $W_1=2\rho_1^{-1}\exp(-\frac{2a(1-c)\mu_1}{A^2})(1+k\exp(\frac{1}{8k})$, $W_2=k\exp(\frac{1}{4k+3}))$, and $W_3=\frac{\sqrt{2(\sigma_1^2+\sigma_2^2)}}{4(1-c)\mu_1}$, and
	 $\rho_1=1-\exp(-\frac{(1-c)\mu_1^2}{2A^2})$.
\end{restatable}

To interpret the results, in Theorem \ref{thm:errorrate},
\newtext{choosing $k=1$ and $c=1-\frac{A^2}{2\mu_0}$ gives $\Pr_0[d=1]\le 2\rho_0^{-1}(1+\exp(\frac{1}{8}))\exp(-b)+\frac{\sqrt{2(\sigma_1^2+\sigma_2^2)}}{2A^2}+\exp(\frac{1}{7})$.}
 Again, the first term is the same as the non-private result $O(\exp(-b))$. The additional $O(\frac{\sqrt{\sigma_1^2+\sigma_2^2}}{A^2})$ term quantifies the cost of privacy.
Since we are instantiating the Gaussian mechanisms with noise parameters $\sigma_1$ and $\sigma_2$ proportional to the sensitivity $2A$ and the privacy parameter $\eps$, the additional error term is reduced to $O(\frac{1}{\eps A})$. This implies the algorithm will incur a larger error rate for stronger privacy guarantees.

\section{Numerical Results}\label{sec:exp}
In this section, we present results from Monte Carlo experiments designed to validate the theoretical results of \PrivSprt.  We only need to validate the statistical properties of \PrivSprt --- sample size and error rates --- since the privacy guarantee holds even in the worst-case over databases and hypotheses.  In Section \ref{exp.bernoulli}, we focus on sequentially testing means of Bernoulli distributions; in Appendix \ref{exp.gaussian}, we provide additional empirical results on testing means of Gaussian distributions. In Appendix \ref{exp.above}, we demonstrate empirically that the classic \AboveThresh\ mechanism does not provide satisfactory performance in terms of sample size and error rates, thus justifying our algorithmic modifications made in \PrivSprt.

\subsection{Testing on Bernoulli Data}\label{exp.bernoulli}

In this section, our experiment focus on Bernoulli data, where $x_1, x_2, \cdots \sim Ber(\theta)$ are sampled
i.i.d. from a Bernoulli distribution with parameter $\theta.$ Monitoring Bernoulli data is one of
the early research in the fully sequential design in clinical trials, see \cite{Arm50}. For
instance, one want to evaluate the effect of a new drug or treatment on the
mortality rate of an unknown infectious disease such as COVID-19 in a sub-population of
groups.

Here we consider two different scenarios that are simple yet useful to shed new lights on real-world applications. One is when the distance between the null hypothesis and the alternative
hypothesis on $\theta$ is large, say, $H_0: \theta = 0.7$ against $H_1: \theta = 0.3,$ e.g., the
effect of a new treatment is expected to be significant to reduce the mortality rate among people
whose age is 65 years or older in a developing country. The other is when the distance between
the null hypothesis and the alternative hypothesis on $\theta$ is small, say, $H_0: \theta = 0.6$
against $H_1: \theta = 0.4$, e.g., the effect of a drug to certain age group with certain diseases
in a developed country. Since $\mu_0=\mu_1$ under this setting, the expected sample sizes under $H_1$ and $H_0$ are identical, and similarly, the Type I error and Type II error are also identical. For simplicity, we will use $E[T]$ and error to denote the expected sample size and the error, respectively.

To obtain an accurate estimate of Type I and Type II errors, we use the importance sampling technique for the Monte Carlo simulations. This is because the estimate of the Type I error based on $n$ independent trials is $n^{-1}\sum_{k=1}^n \frac{f_0(X[1:T])}{f_1(X[1:T])}I(\ell_T(A)+Z_T\ge b+Z_b)$ where the sample $X$ is generated from $f_1$ has much smaller variance compared to the  naive estimate $n^{-1}\sum_{k=1}^n I(\ell_T(A)+Z_T\ge b+Z_b)$ where the sample $X$ is generated from $f_0$.

 We use two $(\eps'/2,\delta=1e-05)$-differentially private Gaussian mechanisms as the noise-adding mechanisms in \PrivSprt, corresponding to $\sigma^2_1=32\log(1.25/\delta)A^2/\eps^2$ and $\sigma^2_2=128\log(1.25/\delta)A^2/\eps^2$.
 Although the log-likelihood ratio is uniformly bounded for Bernoulli data, we invoke the truncation with parameter $A$ because $\mu_0$ and $\mu_1$ are linear with respect to $A$ for Bernoulli data, which makes the validation easier. For each simulation, we repeat the process for $10^5$ times.
 The results are presented in Figure \ref{fig:plot}, which plots the expected sample size $E[T]$ against the log scale of $1/\text{error}$, with varying the privacy parameter $\eps'=0.5, 1, 2$. From this figure, when we want to provide a stronger privacy, i.e., when $\epsilon$ becomes smaller, then we will have larger expected sample sizes for given Type I and Type II error probabilities constraints. This is consistent with our intuition on the tradeoff between privacy and statistical efficiency. 
 
 We also conduct experiments for testing $H_0: \theta = 0.7$ against $H_1: \theta = 0.2,$ when $E_1[T]$ and $E_2[T]$ are not symmetric. We vary this truncation parameter $A=0.05, 0.2, 0.5, 0.7$ in our experiments. For each fixed $A$ and $\eps$, we choose thresholds $a$ and $b$ through Monte Carlo simulation to control the Type I error and Type II error at the same level ($10^{-6}$).The results of these simulations are presented in Table \ref{tbl:bernoulli}. 

 \vspace{-12mm}
 \begin{center}
 	\begin{minipage}{\linewidth}
 		\begin{figure}[H]
 			\centering
 			\subfloat[][Large distance]{\includegraphics[width=.45\textwidth]{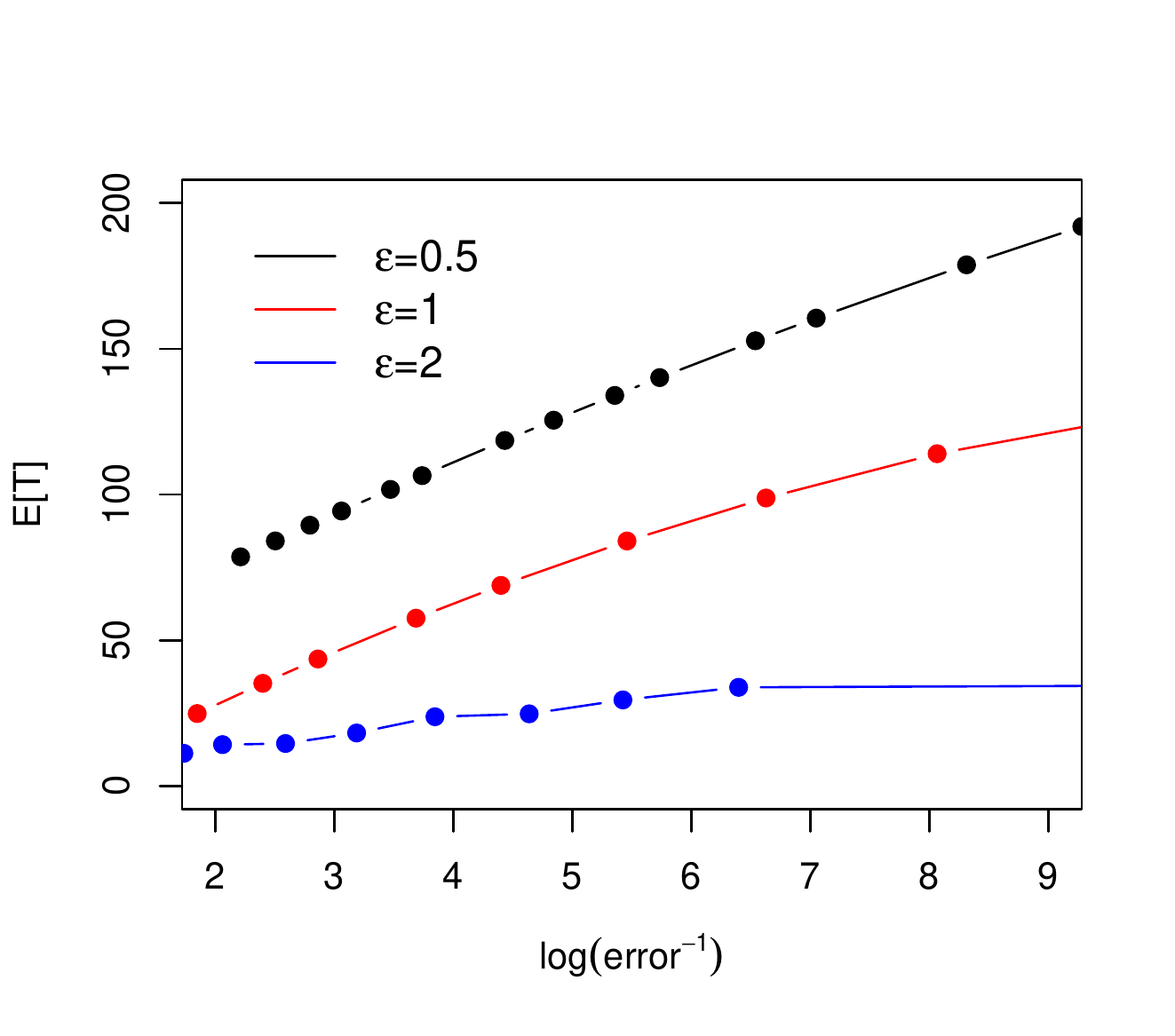}}
 			\subfloat[][Small distance]{\includegraphics[width=.45\textwidth]{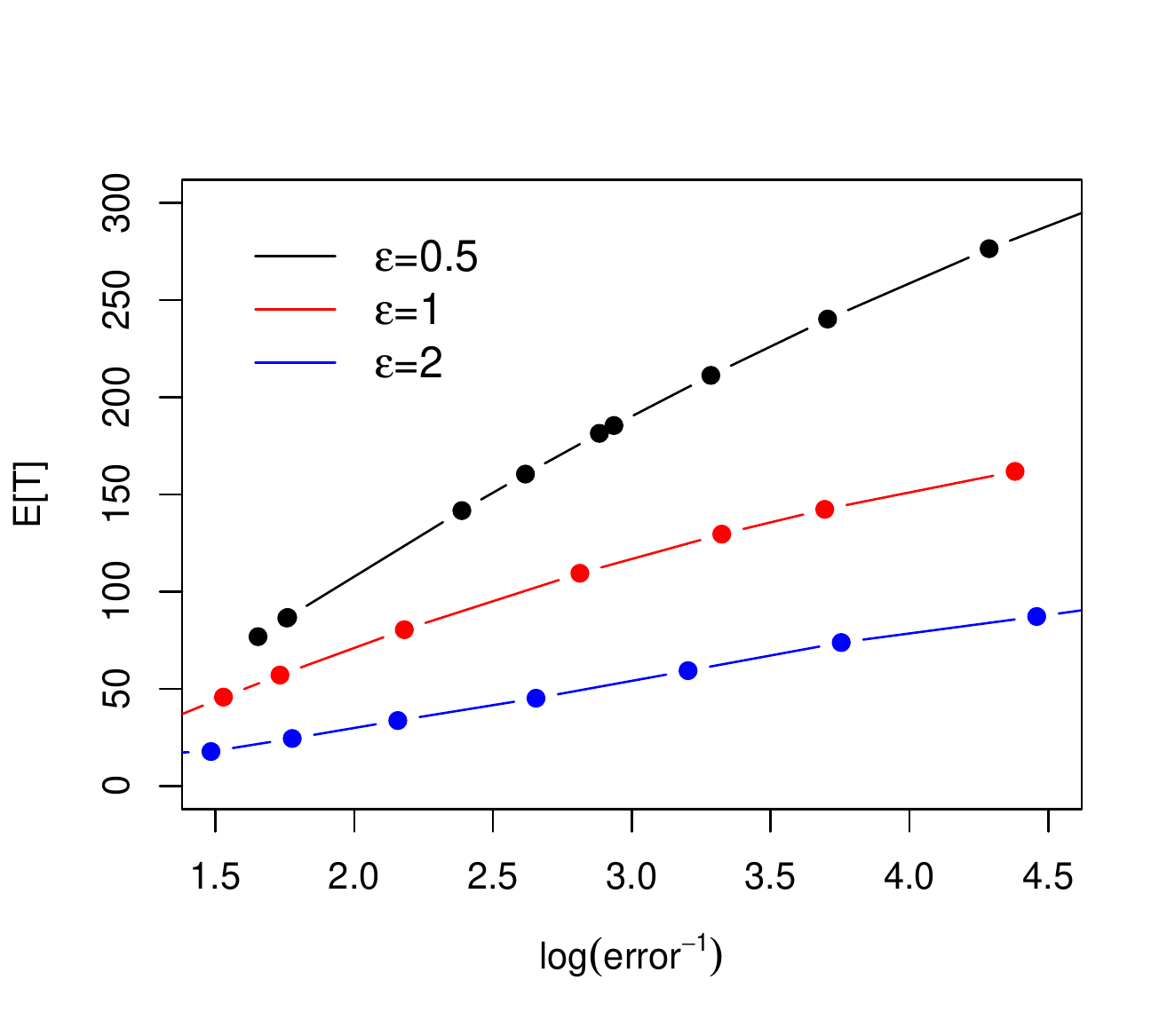}}
 			\caption{\small Three-way trade-off between privacy, expected sample size, and error rate. For large distance (left), we are testing $H_0: \theta = 0.7$ against $H_1: \theta = 0.3$; for small distance (right), we are testing $H_0: \theta = 0.6$
 				against $H_1: \theta = 0.4$.
 			}
 			\label{fig:plot}
 		\end{figure}
 	\end{minipage}
 \end{center}


\begin{table}[h]
	\caption{\small Numerical values of expected sample size under $H_0$ and $H_1$, Type I error and Type II error for testing the Bernoulli parameter. }
	\label{tbl:bernoulli}
	\centering
	\small
	\begin{tabular}{@{}l |l | l | l ||  l |  l  @{}}
		\toprule
		$A$ & $\eps'$ & $a$, $b$ &  error rates & $\mathbb{E}_0[T]$ & $\mathbb{E}_1[T]$ 
		\\ \midrule \midrule
		
		\multirow{3}{*} { 0.05}     &0.5    & 8,7.5  & \multirow{12}{*} {$10^{-6}$}       &   139.662               &   172.89  \\        &1 & 4.3, 4.3 &    & 86.12 & 122.144    \\  &2 & 2.5, 2.5 & &  61.683 & 88.307\\\cmidrule(r){1-3}\cmidrule(r){5-6}
		\multirow{3}{*} { 0.2}     & 0.5 & 32,32 & & 139.456 & 195.504 \\
		& 1 & 16.8, 16.8 & &85.336 & 123.542 \\
		& 2 &9.5,9.5 & & 56.645 & 83.127  \\ \cmidrule(r){1-3}\cmidrule(r){5-6}
		\multirow{3}{*} { 0.5}  & 0.5 & 80,80 & &139.252 & 199.718 \\
		& 1 & 43, 43 & & 88.137&127.986 \\
		& 2 &  25, 25 & & 61.494 & 88.182 \\ \cmidrule(r){1-3}\cmidrule(r){5-6}
		\multirow{3}{*} { 0.7}  & 0.5&  125,120 &  & 173.305 & 227.387 \\
		& 1 &  63,63&  & 95.304 & 136.336\\
		& 2 & 35,35 & & 61.944 & 87.363\\ \cmidrule(r){1-3}\cmidrule(r){5-6}
		& $\infty$ & 16, 16 & & 29.607 & 28.318  \\
		
		\bottomrule
	\end{tabular}
	\vspace{0.1cm}
\end{table}

%
%

Table \ref{tbl:bernoulli} shows three positive results. First, for each fixed privacy parameter $\eps'$, the expected sample sizes are almost the same across varying $A$, and the thresholds are almost linear with respect to $A$. This suggests that the expected sample size $\mathbb{E}_0[T]$ (resp. $\mathbb{E}_1[T]$) is proportional to $a/A$ or $A/a$ (resp. $b/A$ or $A/b$). The parameter $A$ controls a trade-off between how much information is lost from truncation in the log-likelihood ratios and how much noise is added for privacy. Thus expected sample sizes are larger for a larger $A=0.7$ with $\eps'=0.5, 1$, as the additional noise starts to dominate the information provided by the log-likelihood ratios. Second, in our setting, the expectation of the truncated log-likehood ratio $\mu_0=0.6A$ and $\mu_1=0.4A$. We see from Table \ref{tbl:bernoulli} that $\mathbb{E}_0[T]/\mathbb{E}_1[T]$ is roughly $2/3$ for all the cases, which further validates Theorem \ref{thm.sample} that $\mathbb{E}_0[T]$ (resp. $\mathbb{E}_1[T]$) is  $O(a/A)$  (resp. $O(b/A)$).
Third, for each fixed $A$, $a/\mathbb{E}_0[T]$ (resp. $b/\mathbb{E}_1[T]$) decreases as $\eps'$ increases for weaker privacy, which is consistent with Theorem \ref{thm.sample}, because the additional cost does not involve the threshold, and it decreases for weaker privacy.

%

\subsection{Testing on Gaussian Data}\label{exp.gaussian}
In this section, our experiments focus on testing means of Gaussian data, where $x_1, x_2, \ldots \sim N(\mu, 1)$ are sampled i.i.d. from a Gaussian distribution with mean $\mu$. We again consider two different scenarios: large distance between the null and alternative hypotheses on $\mu$ corresponding to $H_0: \mu=0$ against $H_1: \mu=2$, and a small distance between the null and alternative hypotheses on $\mu$ corresponding to $H_0: \mu=0$ against $H_1: \mu=1$. We will denote the expected sample size as $\mathbb{E}[T]$ since $\mathbb{E}_0[T]$ and $\mathbb{E}_1[T]$ are identical for Gaussian data, and similarly, we denote the Type I error and Type II error as errors.
We again use two $(\eps'/2,\delta=1e-05)$-differentially private Gaussian mechanisms as the noise-adding mechanisms. 

In Figure \ref{fig:gauss}, we plot the expected sample size $\mathbb{E}[T]$ against the log scale of $1/$error, and we vary the privacy parameter $\eps'=0.5,1,2$. This experiment is conducted under the setting where we fix the truncation threshold $A=0.5$ and vary the decision threshold $a, b$. For each simulation, we repeat the process $10^5$ times and report average performance. As in the case with Bernoulli data, we see that we experience a larger expected sample size for a given Type I and Type II error constraint as $\eps$ decreases. Additionally, we need fewer samples to distinguish $\mu$ for a large distance regime (Left vs. Right, note the different scales on the y-axes).

\begin{center}
	\begin{minipage}{\linewidth}
		\begin{figure}[H]
			\centering
			\subfloat[][Large distance]{\includegraphics[width=.5\textwidth]{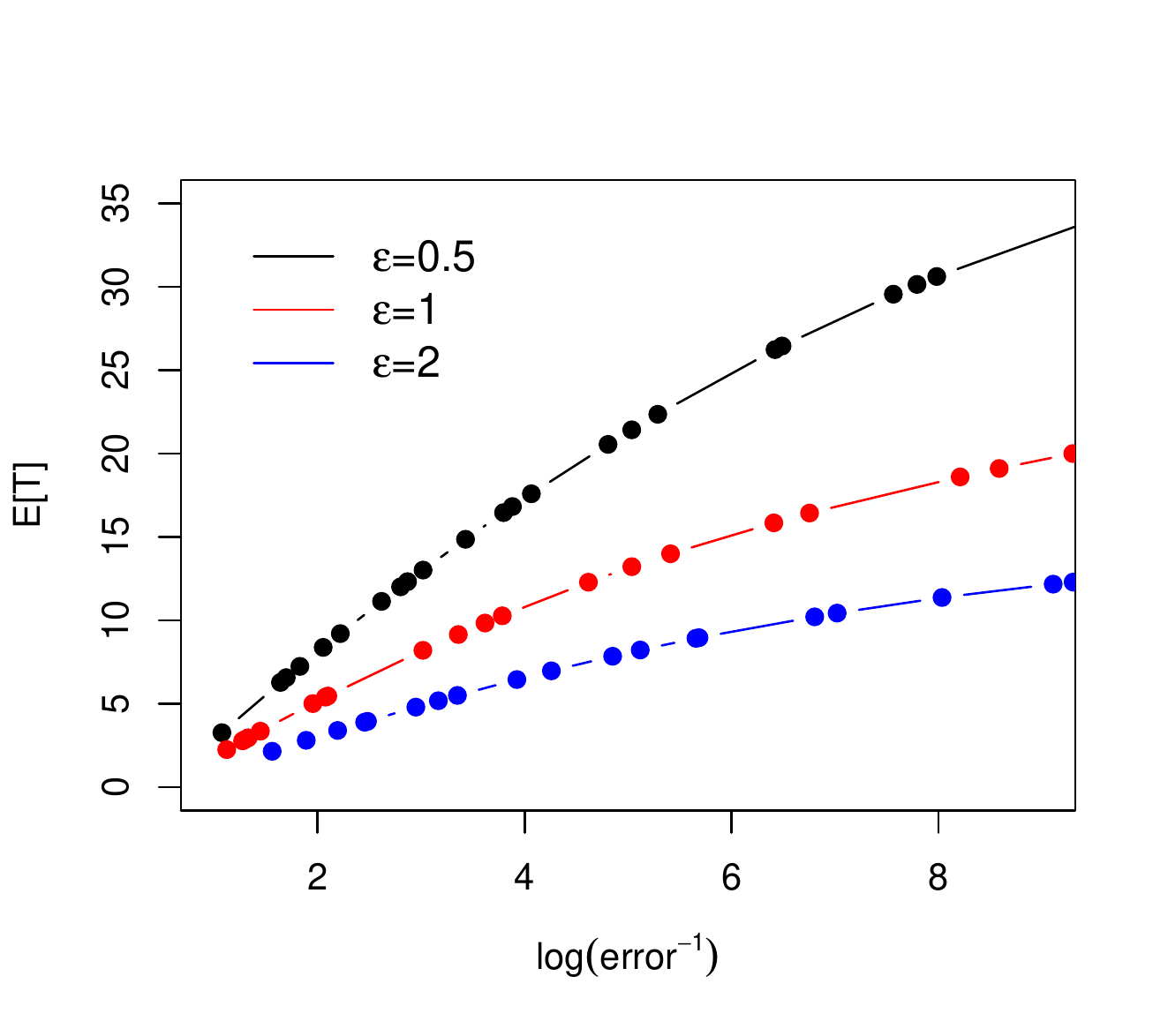}}
			\subfloat[][Small distance]{\includegraphics[width=.5\textwidth]{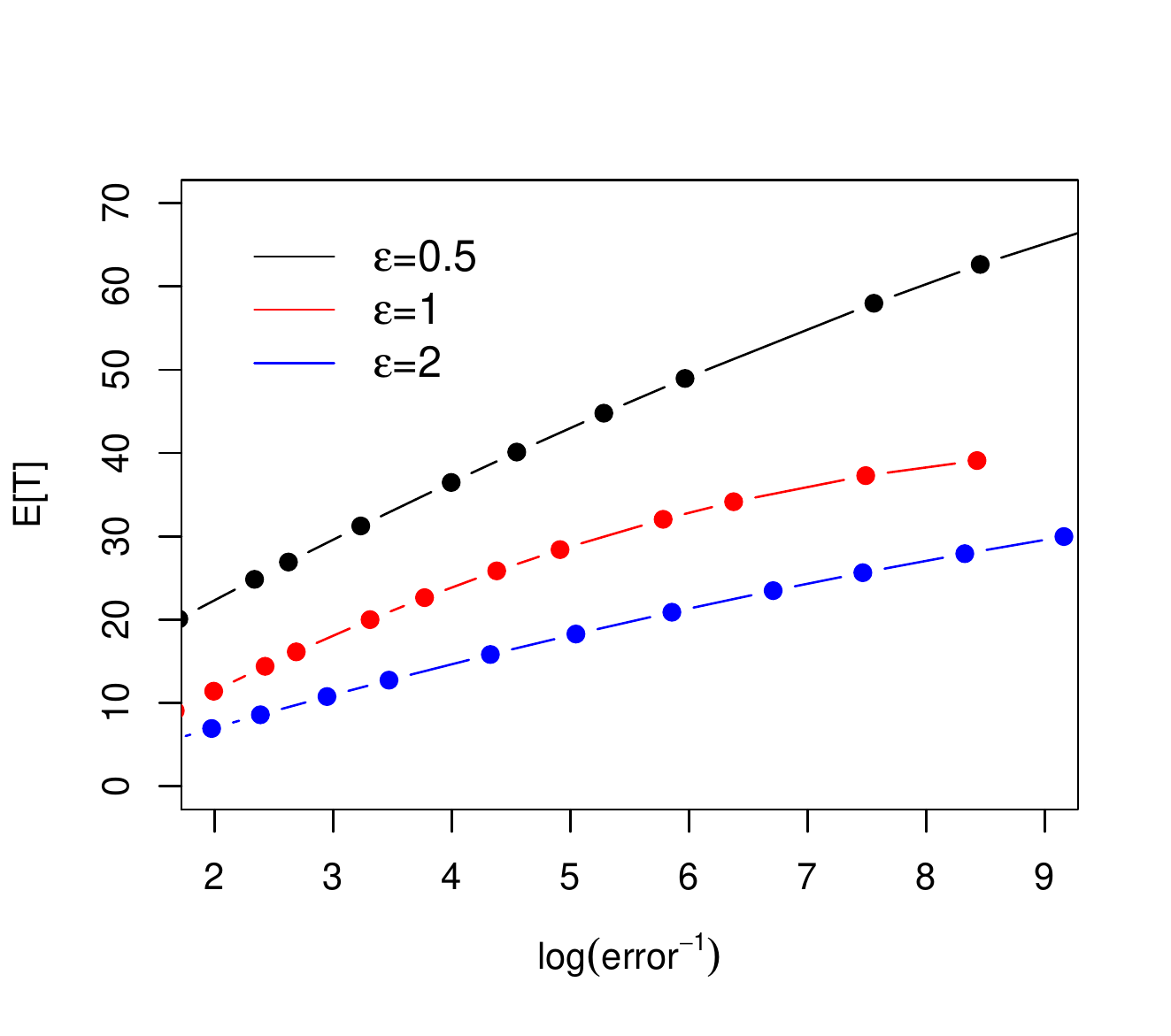}}
			\caption{\small Three-way trade-off between privacy, expected sample size, and error rate. For large distance (left), we are testing $H_0: \mu = 0$ against $H_1: \mu = 2$; for small distance (right), we are testing $H_0: \mu = 0$
				against $H_1: \mu = 1$.
			} 
			\label{fig:gauss}
		\end{figure}
	\end{minipage}
\end{center}

We again conduct experiments to further validate our theoretical results empirically in the Gaussian setting. In Table \ref{tbl:compare} (left),
we vary the truncation parameter $A=0.5,1,2,5$ and the privacy parameter $\eps'=0.5,1,2,\infty$. For each fixed $A$ and $\eps'$, we choose thresholds $a,b$ through Monte Carlo simulation with importance sampling to control Type I and Type II errors at the $0.05$ level. Similar to the results for testing the Bernoulli parameter, the thresholds are almost linear with respect to $A$. The expected sample sizes $\mathbb{E}_0[T]$ and $\mathbb{E}_1[T]$ are almost the same for all $A=0.5, 1, 2$, and $\mathbb{E}_0[T]$ and $\mathbb{E}_1[T]$ increase for a larger $A=5$, because the noise added for privacy dominates the information provided by the log-likelihood ratios. This suggests that a relatively small $A$ is preferred, and as long as $A$ is not too large, it has little impact on the performance. Moreover, we observe that $a/\mathbb{E}_0[T]$ (resp. $b/\mathbb{E}_1[T]$) decreases as $\eps'$ increases for weaker privacy, as the additional cost term in Theorem \ref{thm.sample} decreases for less noise.

%
%
%
%
%

\subsection{Using the standard \AboveThresh.}\label{exp.above}

To compare against the performance of our \PrivSprt, we also conduct experiments for testing means of Gaussian data using the original \AboveThresh\ algorithm with Laplace noise that satisfies $\eps/2$-differential privacy. We now vary the truncation parameter $A=0.5, 1,2,5$, and choose the thresholds such that the Type I and Type II error are below $0.05$. The results are presented in Table \ref{tbl:compare} (right). Table \ref{tbl:compare} shows that using the original \AboveThresh\ algorithm with Laplace noise results in much larger expected sample sizes, given that the Type I and Type II errors are fixed at the $0.05$ level. We note that although the overall privacy cost for \PrivSprt\ is slightly larger, \PrivSprt\ provides a better trade-off between privacy and accuracy.

%
%
%

We also empirically study the overshoot property when adding Laplace noise. We again consider testing $H_0: \theta=0.7$ against $H_1: \theta=0.2$ for Bernoulli data. We choose this setting because $\mu_0\neq\mu_1$, to have a comprehensive view of $\mathbb{E}_0[T]$ and $\mathbb{E}_1[T]$, and the Type I and Type II errors.
We now fix the truncation parameter $A=0.5$, and vary the privacy parameter $\eps=0.5, 1, 2$ and the thresholds $a,b=10,20,40$. The results are presented in Table \ref{tbl:lap}.

\begin{table}[h]
	\caption{Numerical values of expected sample size $\mathbb{E}[T]$, error rates for testing the Gaussian mean using \PrivSprt (left) and the original \AboveThresh\ with Laplace noise (right). The thresholds $a,b$ are chosen to control Type I and Type II error at $0.05$ (within the Monte Carlo simulation errors).}
	\label{tbl:compare}
	\centering
	
	\begin{tabular}{ @{}l |l | l |  l |  l  || l |l | l | l |  l  @{}}
		\toprule
		\multicolumn{5}{c}{\PrivSprt with Gaussian noise} \vline\vline& \multicolumn{5}{c}{\AboveThresh\ with Laplace noise}\\ 
		\toprule
		$A$ & $\eps'$ & $a=b$ & error rates & $\mathbb{E}[T]$ & $A$ & $\eps$ & $a=b$ & error rates &$\mathbb{E}[T]$ 
		\\ \midrule \midrule
		
		\multirow{3}{*} { 0.5}     &0.5   & 9 & \multirow{12}{*} {0.05}          &   12.547           &\multirow{3}{*} { 0.5}     &0.5         & 28  & \multirow{12}{*} {0.05}        &   22.821          \\        &1 & 4 &  & 7.298 & &1 & 12& & 12.621  \\  &2  &2.1& &  4.890 & &2 & 6.5& &  10.482 \\\cmidrule(r){1-3} \cmidrule(r){5-8} \cmidrule(r){10-10}
		\multirow{3}{*} { 1}     & 0.5 & 18 &  & 12.485 &\multirow{3}{*} { 1}     & 0.5 & 59& & 26.032  \\
		& 1 & 8.2 &  & 7.367  && 1 & 29 & & 17.165  \\
		& 2 &4 & & 4.792 && 2 & 15 & & 13.291  \\ \cmidrule(r){1-3} \cmidrule(r){5-8} \cmidrule(r){10-10}
		\multirow{3}{*} { 2}  & 0.5 & 36 &  & 13.333 & \multirow{3}{*} { 2} & 0.5 & 112 & & 23.581   \\
		& 1 & 16 & & 7.460 && 1 & 52 && 15.438 \\
		& 2 &  8 & & 5.000 && 2 & 28 && 12.564  \\ \cmidrule(r){1-3} \cmidrule(r){5-8} \cmidrule(r){10-10}
		\multirow{3}{*} { 5}  & 0.5 & 90 & & 16.943 & \multirow{3}{*} { 5} & 0.5 & 270 && 24.426  \\
		& 1  & 40 & & 10.156  && 1 &140 &&21.067 \\
		& 2 & 98 & & 6.190 && 2 & 70 && 15.872 \\ \cmidrule(r){1-3} \cmidrule(r){5-8} \cmidrule(r){10-10}
		
		& $\infty$ & 2 & & 1.793  \\
		
		\bottomrule
	\end{tabular}
	\vspace{0.1cm}
\end{table}

%
%
%
%

On the theoretical side, we should expect the expected sample size to be $O(b/\mu_1)$ for non-private SPRT. However, we see from Table \ref{tbl:lap} that the expected sample sizes are nonlinear with respect to the thresholds for strong privacy ($\eps=0.5, 1$), which is no longer consistent with the CLT theorem for non-private SPRT.  In contrast, we observe from Table \ref{tbl:bernoulli} that $\mathbb{E}_0[T]$ (resp. $\mathbb{E}_1[T]$) is  $O(a/\mu_0)$  (resp. $O(b/\mu_1)$) in Section \ref{sec:exp} when adding Gaussian noise in \PrivSprt. Intuitively, it appears that the overshoot analysis when adding Laplace noise relies heavily on the additional noise, rather than the statistical information provided by log-likelihood ratios. \newtext{Characterizing the relevant statistical properties when adding Laplace noise requires new tools, which we leave as future work for the privacy and statistics communities.} 

\begin{table}[h]
	\caption{Numerical values of expected sample sizes $\mathbb{E}_0[T]$ and $\mathbb{E}_1[T]$, Type I error and Type II error for testing Bernoulli parameter using the original \AboveThresh\ algorithm with Laplace noise. }
	\label{tbl:lap}
	\centering
	
	\begin{tabular}{@{}l |l | l | l |  l |  l  @{}}
		\toprule
		$a=b$ & $\eps$ & Type I  & Type II & $\mathbb{E}_0[T]$ & $\mathbb{E}_1[T]$
		\\ \midrule \midrule
		
		\multirow{3}{*} { 10}     &0.5        & 0.3634 & 0.3562     &   3.537               &   3.762                  \\        &1  & 0.2184 & 0.3132 & 9.246 & 10.399     \\  &2 & 0.0181 & 0.2185 & 22.317 & 29.035 \\\midrule
		\multirow{3}{*} { 20}     & 0.5  & 0.2577 & 0.1750 & 11.824 & 11.62\\
		& 1  & 0.0235 & 0.0140 &35.353 & 43.121 \\
		& 2 & 1.03e-05 & 3.53e-05 & 55.136 & 77.450\\ \midrule
		\multirow{3}{*} { 40}  & 0.5  & 0.0164 & 0.0266 &51.257 & 66.529\\
		& 1 &8.11e-08& 2.4e-04 & 99.026&144.114 \\
		& 2  & 1.04e-20 & 3.79e-19 & 121.27& 179.172\\
		
		\bottomrule
	\end{tabular}
\end{table}


\bibliography{ref,biblio}
\bibliographystyle{alpha}

\newpage

\appendix

\section{Omitted Proofs}\label{proof}

In this appendix, we provide proofs for our main theorems, which were omitted in the main body of the paper. We restate the theorems here for convenience.

\subsection{Proof of privacy}\label{app.privacy}

\privacy*

\begin{proof}
	We first show that the expectation of the stopping time $T$ is bounded given $Z_A$ and $Z_B$. We instead show the equivalent fact that $P_i(T=\infty)=0$ for $i=0,1$. Define a constant $d=a+b$. If $T=\infty$, then for any positive integer $r$, the following inequalities must hold:
	\begin{equation}\label{ineq1}
	(\sum_{i=kr+1}^{(k+1)r} [\log\frac{f_1(x_i)}{f_0(x_i)}]^A_{-A} + Z)^2<d^2  \quad  k=0, 1, 2, \ldots,
	\end{equation}
	where $Z\sim N(0,\sigma_2^2)$. We can further express $Z$ as a summation of $r$ independent Gaussians $\sum_{i=1}^r Z_i$, and then \eqref{ineq1} is equivalent to 
	\begin{equation}\label{ineq2}
	(\sum_{i=kr+1}^{(k+1)r} ([\log\frac{f_1(x_i)}{f_0(x_i)}]^A_{-A} + Z_i))^2<d^2  \quad  k=0, 1, 2, \ldots.
	\end{equation}
	To prove $P_i(T=\infty)=0$ for $i=0,1$, it is sufficient to show that the probability is zero that \eqref{ineq2} holds for all integer values of $k$. Since the variance of $[\log\frac{f_1(x_i)}{f_0(x_i)}]^A_{-A} + Z_i$ is not zero, and it is bounded below by the variance of $[\log\frac{f_1(x_i)}{f_0(x_i)}]^A_{-A}$, the expected value of $(\sum_{i=1}^j ([\log\frac{f_1(x_i)}{f_0(x_i)}]^A_{-A} + Z_i))^2$ converges to $\infty$ as $j$ goes to $\infty$. Therefore, there exists a positive integer $r$ such that
	\begin{equation}\label{ineq3}
	P[(\sum_{i=1}^j ([\log\frac{f_1(x_i)}{f_0(x_i)}]^A_{-A} + Z_i))^2<r^2]<1.
	\end{equation}
	From \eqref{ineq3} it follows that the probability that \eqref{ineq2} is fulfilled for all values of $k$ up to $\infty$ is equal to zero (using a union bound over all $k$), and thus $P_i(T=\infty)=0$ for $i=0,1$. Hence, $E_i[T| Z_A, Z_b]$ is bounded, and then $E_i[T| Z_A, Z_b]^\gamma$ is bounded. We use the same method to compute the upper bound of $E_i[T| Z_A, Z_b]$ as in the proof in \ref{app.ssproof} with $\sigma_1=0$, and $a$ and $b$ replaced by $a+Z_A$ and $b+Z_B$, respectively. For $i=0, 1$, we denote $E_i[T| Z_A, Z_b]^\gamma$ as $T_A$ and $T_B$, respectively. Since we consider the worst case for privacy, we take the maximum over $T_A$ and $T_B$ in the final bound.
	For Gaussian mechanisms, $\eps_1(\frac{\gamma}{\gamma-1}\alpha)=\frac{\gamma\alpha A^2}{(\gamma-1)\sigma_1^2}$ and $\eps_2(\alpha)=\frac{2\alpha A^2}{\sigma_2^2}$. 
	From inequality \eqref{priv.svt2} in Theorem \ref{thm.privsvt}, it follows that using \GenAboveThresh for truncated log-likelihood ratio queries satisfies $(\frac{\alpha\gamma/(\gamma-1)-1}{\alpha-1} \frac{\alpha A^2}{\sigma_1^2}+\frac{2\alpha A^2}{\sigma_2^2}+\frac{\log \max\{T_A,T_B\}}{\gamma(\alpha-1)})$-RDP, for any $1<\alpha<\infty$.
Then privacy of \PrivSprt\ follows from composition of two parallel instantiations of Algorithm \ref{alg.svt}.  
\end{proof}

\subsection{Proof of sample size}\label{app.ssproof}

\samplesize*
\begin{proof}
	In the proof,  we leverage a critical fact that $\mathbb{E}_i[T]=\sum_{t=0}^\infty \Pr_i(T>t)$ for $i\in\{0,1\}$, and thus relate the expected sample size to the probability of the noisy truncated log-likelihood ratio within the noisy thresholds for each time $t$. Since the event is less probable for a large $t$, we partition the range $[0, \infty)$ into several sub-intervals, and bound the probability in each sub-interval seperately.
	We provide the detailed proof for $E_1[T]$, the proof is the same for $E_0[T]$ with $b$ replaced by $a$ and $\mu_1$ replaced by $\mu_0$.
	\begin{align}
	E_1[T]&=\sum_{t=0}^\infty P_1(T>t) \notag\\
	&\le \sum_{t=0}^\infty P_1(\ell_t+Z_t\le b+Z_b) \notag \\
	&\le \sum_{t=0}^\infty  P_1(\ell_t-t\mu_1\le b-t\mu_1+\delta_t)+P_1(Z_t\le Z_b-\delta_t) \label{ineq.sep}
	\end{align}
	We will bound the first term in \eqref{ineq.sep} as follows. Let $\delta_t=ct\mu_1$, where $c$ is a constant within $(0,1)$. 
	\begin{align}
	&\sum_{t=0}^\infty  P_1(\ell_t-t\mu_1\le b-t\mu_1+\delta_t) \notag\\
	= & \sum_{t=0}^\infty  P_1(\ell_t-t\mu_1\le b-(1-c)t\mu_1) \label{eq.sum} 
	\end{align}
	Let $\gamma$ denote $\frac{b}{(1-c)\mu_1}$, and $m$ denote $(1-c)\mu_1$. We bound the infinite sum in \eqref{eq.sum} by partitioning $[0,\infty]$ into four sub-intervals:
	$$[0, \gamma], \quad (\gamma, \frac{3}{2}\gamma], \quad (\frac{3}{2}\gamma, 2\gamma], \quad (2\gamma, \infty).$$
	Let $S_1, S_2, S_3, S_4$ respectively denote the summation value as the index $t$ ranges over these sub-intervals. 
	When $t\in [0, \gamma]$, we have $b-(1-c)t\mu_1>0$. Since $\ell_t-t\mu_1$ is a mean-zero random variable, we bound $S_1$ by
	\begin{align}
	S_1&=\sum_{t=1}^{[\gamma]} P_1(\ell_t-t\mu_1 \le b-(1-c)t\mu_1) \notag\\
	&\le \sum_{t=1}^{[\gamma]} 1 \le \gamma+1 \label{eq.s1}.
	\end{align}
	When $t>\gamma$, following Hoeffding inequality, we have $P_1(\ell_t-t\mu_1\le b-(1-c)t\mu_1)\le \exp(-\frac{(b-mt)^2}{2tA^2}).$ We will use the following observation as the main tool. For any $i$ and $j$ with $i<j$, we have 
	\begin{align}
	&\sum_{i}^j \exp(-\frac{(b-mt)^2}{2tA^2}) \notag \\
	\le & \sum_{i}^j \exp(-\frac{b^2}{2jA^2}+\frac{bm}{A^2}-\frac{m^2t}{2A^2}) \notag \\
	=& \exp(-\frac{b^2}{2jA^2}+\frac{bm}{A^2}) \sum_{i}^j \exp(-\frac{m^2t}{2A^2}) \notag \\
	=&\exp(-\frac{b^2}{2jA^2}+\frac{bm}{A^2}) \frac{\exp(-\frac{m^2i}{2A^2})-\exp(-\frac{m^2(j+1)}{2A^2})}{1-\exp(-\frac{m^2}{2A^2})} \notag\\
	\le & \rho^{-1}\exp(-\frac{b^2}{2jA^2}+\frac{bm}{A^2})\exp(-\frac{m^2i}{2A^2}) ,\label{eq.tool}
	\end{align}
	where $\rho=1-\exp(-\frac{m^2}{2A^2})$.
	By applying \eqref{eq.tool} to the case where $i=\frac{3}{2}\gamma$ and $j=2\gamma$, we obtain a bound on 
	\begin{equation}\label{eq.s3}
	S_3\le \rho^{-1}.  
	\end{equation}
	Similarly, by applying \eqref{eq.tool} to the case where $i=2\gamma$ and $j=\infty$, we obtain a bound on
	\begin{equation}\label{eq.s4}
	S_4\le \rho^{-1}. 
	\end{equation}
	To bound $S_2$, we further partition the sub-interval $(\gamma, \frac{3}{2}\gamma]$ into $k$ intervals:
	$$(\gamma, \frac{k+2}{k+1}\gamma], \quad \text{and} \quad (\frac{j+2}{j+1}\gamma, \frac{j+1}{j}\gamma],  \text{ for } j=2,\ldots,k.$$
	For the first interval $(\gamma, \frac{k+2}{k+1}\gamma]$, since $b-mt<0$ and $\ell_t-t\mu_1$ is a mean-zero random variable,  we have the simple fact that $P_1(\ell_t-t\mu_1\le b-(1-c)t\mu_1)\le \frac{1}{2}$. Then the summation over the first interval is bounded by $\frac{1}{2(k+1)}\gamma$. By applying \eqref{eq.tool} to the remaining $k-1$ intervals with $i=\frac{j+2}{j+1}\gamma$ and $j=\frac{j+1}{j}\gamma$, we obtain
	\begin{align}
	S_2&\le\frac{1}{2(k+1)}\gamma+\sum_{j=2}^k\sum_{t=[\frac{j+2}{j+1}\gamma]}^{[\frac{j+1}{j}\gamma]}\exp(-\frac{(b-mt)^2}{2tA^2}) \notag\\
	&\le \frac{1}{2(k+1)}\gamma+(k-1)\rho^{-1}, \label{eq.s2}
	\end{align}
	for any $k$. Combining \eqref{eq.s1}, \eqref{eq.s2}, \eqref{eq.s3} and \eqref{eq.s4}, we can bound the first term in \eqref{ineq.sep} by $1+\gamma+\min_{k}\{\frac{1}{2(k+1)}\gamma+(k+1)\rho^{-1}\}$. 
	
	Next we bound the second term in \eqref{ineq.sep}. We will use the fact that $\Pr(N(0,\sigma^2)>x)\le \frac{1}{2}\exp(-\frac{x^2}{2\sigma^2})$ for a Gaussian distribution. 
	\begin{align}
	\Pr(Z_b-Z_t\ge \delta_t) &\le \Pr(N(0, \sigma_1^2+\sigma_2^2)\ge \delta_t) \notag\\
	&\le \frac{1}{2}\exp(-\frac{\delta_t^2}{2(\sigma_1^2+\sigma_2^2)})
	\end{align}
 We now consider the sum of these terms over all $t$:
	\begin{align}
	\sum_{t=0}^{\infty}\Pr(Z_b-Z_t\ge \delta_t) &\le \frac{1}{2} \sum_{t=0}^\infty \exp(-\frac{(ct\mu_1)^2}{2(\sigma_1^2+\sigma_2^2)}) \\
	&=\frac{\sqrt{2(\sigma_1^2+\sigma_2^2)}}{2(1-c)\mu_1}\sum_{t=0}^\infty\exp(-t^2)\\
	&\le \frac{3\sqrt{2(\sigma_1^2+\sigma_2^2)}}{4(1-c)\mu_1}.
	\end{align}
	
	We combine these to derive the final bound as desired:
	$E_1[T] \le 1+\frac{b}{(1-c)\mu_1}+\min_{k}\min_c\{\frac{1}{2(k+1)}\frac{b}{(1-c)\mu_1}+(k+1)\rho^{-1}+ \frac{3\sqrt{2(\sigma_1^2+\sigma_2^2)}}{4(1-c)\mu_1}\}.$ The bound on $E_0[T]$ follows by symmetry, with $b$ replaced by $a$, and $\mu_1$ replaced by $\mu_0$.

\end{proof}

\subsection{Proof of error rate}\label{app.errorproof}

\errorrate*
\begin{proof}
The proof of error rates is based on a brute force estimation of the error probabilities: We can write the Type I error $\Pr_0[d=1]$  as a sum of probabilities of the noisy log-likelihood ratio being above the noisy threshold at time $t$ for all $t>0$: $\sum_{t=1}^\infty \Pr_0[\ell_t(A)+Z_t>b+Z_b \; \land \; T=t]$. We then partition the range $[1,\infty)$ into several sub-intervals and analyze them separately. 
We provide the detailed proof for $P_0(d=1)$. The proof is the same for $P_0(d=1)$ with $b$ replaced by $a$, and $\mu_0$ replaced by $\mu_1$. To start, we have the following brute force estimation:
	\begin{align}
	P_0(d=1)&=P_0(\ell(T)+Z_T>b+Z_b) \notag\\
	&=\sum_{t=1}^\infty P_0(T=t, \ell(t)+Z_t>b+Z_b) \notag\\
	&\le \sum_{t=1}^\infty P_0(\ell(t)+Z_t>b+Z_b) \notag\\
	&\le \sum_{t=1}^\infty P_0(\ell(t)+t\mu_0>b+t\mu_0-\delta_t)+\Pr(Z_t-Z_b>\delta_t). \label{eq.split} 
	\end{align}
	We choose $\delta_t=ct\mu_0$, where $c$ is a constant within $(0,1)$. To simply the notation, we let $\gamma$ denote $\frac{b}{(1-c)\mu_0}$, and $m$ denote $(1-c)\mu_0$.
	We bound the first term in \eqref{eq.split} using similar technique as in the proof of Theorem \ref{thm.sample}. We partition $[1,\infty)$ into four sub-intervals:
	$$[0, \frac{1}{2}\gamma], \quad (\frac{1}{2}\gamma, \gamma], \quad (\gamma, 2\gamma], \quad (2\gamma, \infty).$$
	We will use the following observation as the main tool. For any $i$ and $j$ with $i<j$, we have
	\begin{align}
	&\sum_{i}^j \exp(-\frac{(b+mt)^2}{2tA^2}) \notag \\
	\le & \sum_{i}^j \exp(-\frac{b^2}{2jA^2}-\frac{bm}{A^2}-\frac{m^2t}{2A^2}) \notag \\
	=& \exp(-\frac{b^2}{2jA^2}-\frac{bm}{A^2}) \sum_{i}^j \exp(-\frac{m^2t}{2A^2}) \notag \\
	\le & \rho^{-1}\exp(-\frac{b^2}{2jA^2}-\frac{bm}{A^2})\exp(-\frac{m^2i}{2A^2}) ,\label{eq.tool2}
	\end{align}
	where $\rho=1-\exp(-\frac{m^2}{2A^2})$. 
	By applying \eqref{eq.tool2} to the case that $i=1$ and $j=\frac{1}{2}\gamma$, we obtain 
	\begin{align}
	S_1&\le \rho^{-1}\exp(-\frac{2bm}{A^2})\exp(-\frac{m^2}{2A^2}) \notag\\
	&\le \rho^{-1}\exp(-\frac{2bm}{A^2}) \label{eq.S1},
	\end{align}
	where \eqref{eq.S1} follows from the fact that $\exp(-\frac{m^2}{2A^2})<1$. Similarly, we have
	\begin{equation}\label{eq.S4}
	S_4\le \rho^{-1}\exp(-\frac{2bm}{A^2}).
	\end{equation}
	We further partition $(\frac{1}{2}\gamma, \gamma]$ into $k$ sub-intervals $(\frac{k+j-1}{2k}\gamma, \frac{k+j}{2k}\gamma]$ for $j=1,2,\ldots, k$. We have 
	\begin{align}
	S_2&\le \sum_{j=1}^{k} \rho^{-1} \exp(-\frac{bm}{A^2}\frac{k}{k+j}-\frac{bm}{A^2}-\frac{bm}{2A^2}\frac{k+j-1}{2k}) \notag\\
	& = \sum_{j=1}^{k} \rho^{-1}\exp(-\frac{bm}{2A^2}(2+(\frac{2k}{k+j}+\frac{k+j-1}{2k}))) \notag\\
	&\le \sum_{j=1}^{k} \rho^{-1}\exp(-\frac{bm}{2A^2}(2+2-\frac{1}{2k})) \label{eq.opt} \\
	& = k\rho^{-1}\exp(-\frac{2bm}{A^2}\frac{8k-1}{8k}). \label{eq.S2}
	\end{align}
	Similarly, we have 
	\begin{equation}\label{eq.S3}
	S_3 \le k\rho^{-1}\exp(-\frac{2bm}{A^2}\frac{4k+3}{4k+3}).
	\end{equation}
	
	Next we bound the error from the added Gaussian noise.
	\begin{align}
	\sum_{t=1}^{\infty}\Pr(Z_t-Z_b\ge \delta_t) &\le \frac{1}{2} \sum_{t=0}^\infty \exp(-\frac{(ct\mu_0)^2}{2(\sigma_1^2+\sigma_2^2)}) \\
	&=\frac{\sqrt{2(\sigma_1^2+\sigma_2^2)}}{2(1-c)\mu_0}\sum_{t=1}^\infty\exp(-t^2)\\
	&\le \frac{\sqrt{2(\sigma_1^2+\sigma_2^2)}}{4(1-c)\mu_0} \label{eq.noise2}
	\end{align}
	
	Combining \eqref{eq.S1}, \eqref{eq.S2}, \eqref{eq.S3}, \eqref{eq.S4} and \eqref{eq.noise2}, we obtain
	\begin{equation}
	P_0(d=1)\le \min_k \min_c \{ 2\rho^{-1}\exp(-\frac{2b(1-c)\mu_0}{A^2})(1+k\exp(\frac{1}{8k}) \notag+k\exp(\frac{1}{4k+3}))+\frac{\sqrt{2(\sigma_1^2+\sigma_2^2)}}{4(1-c)\mu_0} \}
	\end{equation}
	
\end{proof}

\end{document}